\documentclass[12pt]{article}

\usepackage[margin=1.05in]{geometry}
\usepackage{amsmath}
\usepackage{amssymb}
\usepackage{amsthm}
\usepackage{dsfont}
\usepackage{graphicx} %
\usepackage{microtype}
\usepackage{caption}

\usepackage{seqsplit}

\usepackage{authblk}

\usepackage[cal=euler]{mathalpha}

\usepackage{setspace}
\allowdisplaybreaks
\makeatletter
\g@addto@macro\normalsize{%
    \setlength\abovedisplayskip{10pt}
    \setlength\belowdisplayskip{10pt}
    \setlength\abovedisplayshortskip{10pt}
    \setlength\belowdisplayshortskip{10pt}
}
\makeatother

\usepackage[dvipsnames]{xcolor}
\usepackage[backref=page]{hyperref}
\hypersetup{
    colorlinks=true,
    linkcolor=cyan!80!black,
    citecolor=MidnightBlue,
    urlcolor=magenta,
}

\usepackage{tikz}
\usetikzlibrary{matrix, positioning, decorations.pathreplacing, calc, arrows.meta}
\usepackage{natbib}

\usepackage{apptools}
\AtAppendix{\counterwithin{lemma}{section}}

\newcommand{\bbR}{\mathbb{R}}

\newcommand{\ba}{\mathbf{a}}
\newcommand{\bb}{\mathbf{b}}

\newcommand{\bh}{\mathbf{h}}
\newcommand{\bx}{\mathbf{x}}

\newcommand{\bz}{\mathbf{z}}

\newcommand{\balpha}{\boldsymbol{\alpha}}

\newcommand{\cC}{\mathcal{C}}

\newcommand{\ind}{\mathds{1}}

\theoremstyle{plain}
\newtheorem{theorem}{Theorem}[section]

\theoremstyle{definition}

\newtheorem{definition}{Definition}[section]

\theoremstyle{remark}
\newtheorem{remark}{Remark}[section]

\let\tilde\widetilde

\title{\textbf{Approximation Theory for Neural Networks: Old and New}}

\author{Soumendu Sundar Mukherjee}
\author{Himasish Talukdar}
\affil{Statistics and Mathematics Unit, Indian Statistical Institute, Kolkata}
\affil[ ]{\texttt{ssmukherjee@isical.ac.in, talukdar.himasish@gmail.com}}
\date{}

\begin{document}

\maketitle

\begin{abstract}
Universal approximation theorems provide a mathematical explanation for the expressive power of neural networks. They assert that, under mild conditions on the activation function, feedforward neural networks are dense in broad function classes, such as continuous functions on compact subsets of $\mathbb{R}^d$, $L^p$ spaces, or Sobolev spaces. Over the past four decades, these qualitative universality results have evolved into a rich quantitative theory addressing approximation rates, parameter efficiency, and the role of architectural features such as depth and width. This survey presents several glimpses into this theory. We review classical density results for single-hidden-layer networks, as well as quantitative bounds that relate approximation error to network size and smoothness assumptions on target functions. Particular emphasis is placed on depth--width trade-offs and on results demonstrating that deeper architectures can achieve superior parameter efficiency for structured function classes. In addition to standard feedforward neural networks, we also review recent developments on Kolmogorov--Arnold Networks (KANs), which offer an alternative architectural paradigm and whose approximation-theoretic properties have begun to attract significant theoretical attention.
\end{abstract}

\smallskip

\section{Introduction}\label{sec:intro}
Neural networks are the dominant function approximation paradigm in modern machine learning and artificial intelligence. In essence, neural networks are classes of parameterised nonlinear functions, constructed by composing affine transformations with nonlinear activation functions. In their simplest \emph{feedforward} form, a neural network maps an input $\bx \in \mathbb{R}^d$ to an output through successive layers
\[
    \bz_k = \sigma(W_k \bz_{k-1} + b_k), \qquad \bz_0 = \bx,
\]
where $W_k$ and $b_k$ are learnable weight matrices and bias vectors, and $\sigma$ is a \emph{fixed nonlinear activation function} applied componentwise. The final layer is often taken to be linear, producing an output such as $\balpha^\top \bz_K$ (see Figure~\ref{fig:FNN-schematic} for a schematic representation of this architecture). By varying the network width, depth, and parameters, neural networks define rich families of functions. A natural question then is:
\begin{quote}
    Given a target class of functions, are a given class of neural networks dense in it? How efficient and parsimonious can the representation be?
\end{quote}

Universal Approximation Theorems (UATs) answer the above questions by making the expressivity of neural nets mathematically precise. They show that, under mild assumptions on the activation function $\sigma$, neural networks are dense in suitable function spaces, typically continuous functions on compact subsets of $\mathbb{R}^d$ with respect to the uniform norm, $L^p$ spaces, or Sobolev spaces. In this sense, neural networks can approximate most reasonable target functions on compact domains to arbitrary accuracy.

It is important to distinguish between qualitative and quantitative forms of universality. Classical UATs established density, guaranteeing the existence of approximating networks. Subsequent works have refined these results by providing approximation rates and by analysing how architectural features such as depth and width influence representational efficiency. In particular, while $1$-layer networks (i.e., networks with a single hidden layer) are already universal under mild conditions, deeper architectures can in many cases achieve the same approximation accuracy with substantially fewer parameters. This survey focuses on these approximation-theoretic aspects of neural networks, rather than on questions of optimization or statistical generalization. While the latter two topics are of paramount importance for deep learning practice, approximation-theoretic results inform neural network design by revealing what is achievable and what is not.

While our primary focus will be on feedforward neural networks (FNNs)\footnote{We shall use the terms \emph{feedforward neural networks} and \emph{multilayer perceptrons} interchangeably.}, we will also discuss the recently popularised Kolmogorov--Arnold Networks (KANs), whose design is inspired by the Kolmogorov--Arnold representation theorem. In contrast to standard FNNs, where nonlinearities are fixed and applied at the nodes, KANs place learnable univariate functions on the edges, thereby replacing fixed activations with adaptive functional transformations. From the perspective of universal approximation, it is therefore natural to ask how KANs compare to classical FNNs in terms of approximation rates and parameter efficiency. A definitive answer is yet to be found!

\textbf{Note for readers.} For this survey, although we shall define everything, some basic familiarity with neural network vocabulary will be helpful. We shall also assume that the reader has basic knowledge of the Fourier transform and a few fundamental results from functional analysis (e.g., the Riesz representation theorem, the Hahn--Banach theorem, etc.). Even if the reader chooses to skip the parts where these are used, a certain level of mathematical maturity is desirable for smooth sailing.

\begin{figure}
\begin{minipage}{.38\textwidth}
    \centering
    \begin{tikzpicture}[scale=2]
    \draw[->] (-0.2,0) -- (1.5,0) node[right] {$x_1$};
    \draw[->] (0,-0.2) -- (0,1.5) node[above] {$x_2$};
    \draw[step=1cm,gray,very thin,dashed] (0,0) grid (1,1);

    \filldraw[black] (0,0) circle (2pt) node[below left] {$(0,0)$};
    \filldraw[black] (1,1) circle (2pt) node[above right] {$(1,1)$};

    \draw[thick] (0.9, -0.1) -- (1.1, 0.1);
    \draw[thick] (0.9, 0.1) -- (1.1, -0.1);
    \node[below right] at (1,0) {$(1,0)$};

    \draw[thick] (-0.1, 0.9) -- (0.1, 1.1);
    \draw[thick] (-0.1, 1.1) -- (0.1, 0.9);
    \node[above left] at (0,1) {$(0,1)$};
\end{tikzpicture}
    \captionof{figure}{Linear non-separability of XOR is visually evident in $\mathbb{R}^2$, where no single line can isolate the classes. Here,  circles represent output $0$, crosses $1$.}
    \label{fig:xor-linear-non-separability}
\end{minipage}%
\hspace{5mm}
\begin{minipage}{0.58\textwidth}
    \centering
    \begin{tikzpicture}[
    plain/.style={
      draw=none,
      fill=none,
    },
    net/.style={
      matrix of nodes,
      nodes={
        draw,
        circle,
        inner sep=8pt
      },
      nodes in empty cells,
      column sep=2cm,
      row sep=0.5cm
    },
    >=stealth
]

\node[circle, fill=green!20, draw, minimum size=1cm] (x1) at (0,-1) {$x_1$};
\node[circle, fill=green!20, draw, minimum size=1cm] (x2) at (0,-3) {$x_2$};

\node[circle, fill=blue!20, draw, minimum size=1cm] (h11) at (3.5,0) {$h_{11}$};
\node[circle, fill=blue!20, draw, minimum size=1cm] (h12) at (3.5,-4) {$h_{12}$};

\node[circle, fill=red!20, draw, minimum size=1cm] (y) at (6,-2) {$y$};

\node[below=2mm of h11] (b11) {$1.5$};
\node[below=2mm of h12] (b12) {$-0.5$};
\node[below=2mm of y] (b_out) {$-1.5$};

\draw[->] (x1) -- (h11) node[midway, above] {$-1$};
\draw[->] (x1) -- (h12) node[midway, above] {$1$};
\draw[->] (x2) -- (h11) node[midway, above] {$-1$};
\draw[->] (x2) -- (h12) node[midway, above] {$1$};

\draw[->] (h11) -- (y) node[midway, above] {$1$};
\draw[->] (h12) -- (y) node[midway, below] {$1$};

\draw[->, thick] (y.east) -- ++(2,0) node[midway, above] {};

\node[above] at (4.5, 0.3) {NAND};
\node[above] at (4.35, -4.7) {OR};
\node[above] at (6.9, -1.7) {AND};
\end{tikzpicture}
    \captionof{figure}{A multilayer perceptron representing XOR. The numbers beside the edges denote the corresponding weights and the number below nodes denote the biases. The hidden layer transforms the input space into a feature space where the classes become linearly separable. The two hidden units/neurons representing the NAND and OR gates. The outputs of these are then passed though another neuron that implements the AND gate.
    }
    \label{fig:xor-mlp}
\end{minipage}
\end{figure}

\subsection{Early neural networks and the role of depth}

Early mathematical models of neural networks date back to the work of \cite{mcculloch1943logical}, who proposed simple threshold units as models of biological neurons. Building on this idea, \cite{rosenblatt1958perceptron} introduced the perceptron, a mapping of the form
\[
    \mathbf{x} \mapsto \ind(\mathbf{w}^\top \mathbf{x} + b \ge 0),
\]
which is essentially a linear classifier\footnote{We denote by $\ind(\mathcal{E})$ the indicator of a condition $\mathcal{E}$. We also use $\ind_A$ to denote the indicator function of a set $A$.}. Geometrically, a perceptron partitions the input space by a single affine hyperplane. Therefore, while perceptrons can successfully represent linearly separable outputs, they fail on simple linearly non-separable problems. A classical example is the XOR gate, which outputs $1$ if exactly one of the two binary inputs is $1$, and $0$ otherwise. As illustrated in Figure~\ref{fig:xor-linear-non-separability}, no single affine hyperplane in $\mathbb{R}^2$ can separate the two classes (represented by the outputs), and hence a single perceptron cannot represent XOR.

This limitation motivates the introduction of multilayer architectures. By adding a hidden layer with several perceptron units (also called neurons), the input is first mapped into a feature space in which the classes become linearly separable. In this way, even a simple network with one hidden layer (see Figure~\ref{fig:xor-mlp} for an example) can represent the XOR gate exactly\footnote{In fact, \cite{baum1988capabilities} showed that one can separate $N$ arbitrary points in $\bbR^d$ into two classes by a perceptron network with one hidden layer consisting of $\lceil N / d \rceil$ neurons. This result is optimal, in that there are examples of dichotomies not realizable with fewer neurons.}. This example illustrates a fundamental principle: depth allows neural networks to construct intermediate representations, enabling them to model functions that are inaccessible to purely linear threshold units. 

\subsection{A selective review of the history of UATs for FNNs}

One of the first general UATs was due to \cite{cybenko1989approximation}, who proved that $1$-layer networks with continuous sigmoidal activations are dense in the space of continuous functions on compact subsets of $\mathbb{R}^d$. \cite{funahashi1989approximate} arrived at the same conclusion for continuous and monotone sigmoidal activations. \cite{hornik1989multilayer} also proved the same density result using monotone sigmoidal activations. Remarkably, these three almost concurrent works use quite different techniques. Later, \cite{hornik1991approximation} showed that taking bounded, non-constant, continuous activations suffices for approximating continuous functions on compact sets. He also obtained density results with respect to $L^p$ norms for bounded non-constant activations. The precise characterisation of which activations suffice for universal approximation of continuous functions on compact sets, namely, any continuous, non-polynomial function, was subsequently given by \cite{leshno1993multilayer}. These works collectively establish that shallow networks with a single hidden layer possess remarkable expressive power in principle. For a more detailed commentary on these early works, and further references, we refer the reader to the survey \cite{pinkus1999approximation}.

While density ensures the existence of approximating networks, it does not quantify the number of neurons required. \cite{barron1993universal} addressed this gap by identifying a natural function class, characterised by a certain bounded Fourier norm condition, for approximation by $1$-layer networks and showing that such functions can be approximated in $L^2$ over any closed ball at a rate $\mathcal{O}(m^{-1/2})$, where $m$ is the number of neurons. \cite{makovoz1996random} sharpened this analysis by proving an improved rate of $\mathcal{O}(m^{-1/2 - 1/(2d)})$ in $L^2$, with matching lower bounds, and extended the approximation results to $L^p$ norms. \cite{mhaskar1996neural} established optimal-order approximation rates on Sobolev classes, showing that $1$-layer networks achieve the best possible approximation error for functions with a prescribed number of derivatives. More recently, \cite{de2021approximation} derived UATs for $1$-layer networks with $\tanh$ activation in Sobolev norms; these results are particularly relevant to scientific computing applications such as Physics-Informed Neural Networks (PINNs) for solving PDEs.

Due to the dramatic empirical success of deep learning \citep{lecun2015deep}, a central question in modern neural network approximation theory is to understand what expressivity advantages deeper architectures bring over shallower ones. \cite{montufar2014number} gave an early quantification of this for ReLU networks by showing that the number of linear regions induced in the input space grows exponentially with depth under mild width assumptions. \cite{bianchini2014complexity} compared deep and shallow networks with the same number of hidden neurons in terms of certain topological complexity measures of the functions they represent, establishing that deeper architectures can realise functions of higher complexity. \cite{telgarsky2016benefits} gave a striking illustration: there are explicit functions efficiently representable by deep networks of bounded width that require exponentially many neurons in shallower architectures. \cite{eldan2016power} constructed functions representable as $2$-layer networks of polynomial width in the input dimension that cannot be approximated by $1$-layer networks unless their width is exponential in the input dimension. \cite{poole2016exponential, raghu2017expressive} further studied the geometric complexity of network mappings, demonstrating that depth exponentially amplifies the trajectory length and curvature of input signals. On the quantitative side, various authors have established explicit error bounds for approximation by deep ReLU networks, demonstrating that depth can substantially reduce the number of neurons needed to achieve a specified accuracy \citep{yarotsky2017error, yarotsky2018optimal, petersen2018optimal, kohler2021rate, daubechies2022nonlinear, siegel2023optimal}. Another important line of work concerns characterising the minimum width required for approximation by deep networks, beginning with \cite{lu2017expressive}. These results collectively give ample mathematical support to the view that depth fundamentally increases the representational capacity of neural networks.

\subsection{Other surveys}
Universal approximation has been studied extensively over the past four decades, and several works provide broad overviews of the classical theory. The already-mentioned survey \cite{pinkus1999approximation} offers a comprehensive account of the approximation theory for multilayer perceptrons developed up to the late 1990s. This work remains a foundational reference for the approximation-theoretic perspective on neural networks. A slightly earlier survey covering some of the same topics is \cite{scarselli1998universal}.

Since these early surveys, the literature has expanded significantly, particularly in the direction of quantitative approximation rates, expressivity and separation results, and depth--width trade-offs. However, many of these developments are dispersed across specialised research articles rather than consolidated expository treatments. One aim of the present survey is therefore to present a short and readable gateway to modern results. Due to the vastness of the literature, we had to be selective. The readers can find further relevant works by consulting the bibliographies of the papers cited here. Recommended is also the modern and much more expansive survey of \cite{devore2021neural}, which focuses on comparing neural network approximation with other traditional methods such as approximation by polynomials, wavelets or splines.

\subsection{Organisation}
The rest of this survey is organised as follows. In Section~\ref{sec:fnn} we look at UATs for FNNs, beginning with basic density results and then moving on to quantitative approximation bounds; also discussed are modern quantitative results highlighting the role of depth and width. In Section~\ref{sec:kan}, we discuss KANs and their approximation theory. We end the survey with some concluding remarks in Section~\ref{sec:conc}.

\section{Feedforward Neural Networks}\label{sec:fnn}
\subsection{Formal definitions}\label{sec:defn}
We begin with some basic definitions. We are interested in mappings from $\mathbb{R}^d$ to $\mathbb{R}$ of the form
\begin{equation}\label{eq:defn_single_layer}
    \mathbf{x} \;\mapsto\; \sum_{j=1}^m a_j \,\sigma(\mathbf{w}_j^\top \mathbf{x} + b_j),
\end{equation}
where $\mathbf{w}_j \in \mathbb{R}^d$ and $a_j, b_j \in \mathbb{R}$ for $j=1,\ldots,m$. The function $\sigma:\mathbb{R}\to\mathbb{R}$ is called the \emph{activation function}. Typical examples include the ReLU function $z\mapsto \max\{0,z\}$ and the sigmoid function $z\mapsto \frac{1}{1+e^{-z}}$. The map $\mathbf{x} \;\mapsto\; \sigma(\mathbf{w}_j^\top \mathbf{x} + b_j)$ defines a \emph{neuron}.

In this setting, $\mathbf{x}$ denotes the input data, and the collection $((a_j,\mathbf{w}_j,b_j))_{j=1}^m$ forms the set of trainable parameters, which are optimised using observed data. The integer $m$ is called the \emph{width} of the network.

It is convenient to express \eqref{eq:defn_single_layer} in matrix–vector form. Let $W$ be the $m\times d$ matrix whose $j$-th row is $\mathbf{w}_j^\top$. For $\mathbf{y}=(y_1,\ldots,y_m)^\top\in\mathbb{R}^m$, define
\[
\sigma(\mathbf{y}) := (\sigma(y_1),\ldots,\sigma(y_m))^\top,
\]
so that $\sigma$ acts coordinatewise. Writing $\mathbf{a}=(a_1,\ldots,a_m)^\top$ and $\mathbf{b}=(b_1,\ldots,b_m)^\top$, the mapping in \eqref{eq:defn_single_layer} can be written compactly as
\[
\mathbf{x} \;\mapsto\; \mathbf{a}^\top \sigma(W\mathbf{x}+\mathbf{b}).
\]

Under our convention, a network of the form \eqref{eq:defn_single_layer} is called a \emph{$1$-layer neural network}. It consists of a single nonlinear \emph{hidden layer} followed by a linear combination of its outputs\footnote{Many authors call this architecture a 2-layer network, accounting for the output layer which also involves learnable parameters. However, we shall only count the hidden layers, as they are what give neural networks their expressivity. Under this convention, the perceptron would be a $0$-layer neural network.}.

\begin{figure}[!t]
\centering
\begin{tikzpicture}[
    scale=0.95,
    neuron/.style={circle, draw, minimum size=16pt, inner sep=0pt, align=center},
    input/.style={neuron, fill=green!10, draw=green!60!black},
    hidden/.style={neuron, fill=blue!10, draw=blue!60!black},
    output/.style={neuron, fill=red!10, draw=red!60!black},
    arrow/.style={-{Stealth[scale=1.0]}, shorten >=1pt, semithick, color=gray!80},
    label text/.style={font=\footnotesize\sffamily, color=gray!70}
]

    \def\layersep{2.5cm} %
    \def\nodesep{1.2cm}  %

    \node[input] (I-1) at (0,0) {$x_1$};
    \node[input] (I-2) at (0,-\nodesep) {$x_2$};
    \node (I-dots) at (0,-2*\nodesep) {$\vdots$};
    \node[input] (I-n) at (0,-3*\nodesep) {$x_n$};
    
    \node[above=0.2cm of I-1] {Input };

    \node[hidden] (H1-1) at (\layersep, \nodesep*0.5) {};
    \node[hidden] (H1-2) at (\layersep, -\nodesep*0.5) {};
    \node (H1-dots) at (\layersep, -1.5*\nodesep) {$\vdots$};
    \node[hidden] (H1-m) at (\layersep, -3.5*\nodesep) {};
    
    \node[above=0.2cm of H1-1, align=center] {Layer 1};

    \node[hidden] (H2-1) at (2*\layersep, \nodesep*0.5) {};
    \node[hidden] (H2-2) at (2*\layersep, -\nodesep*0.5) {};
    \node (H2-dots) at (2*\layersep, -1.5*\nodesep) {$\vdots$};
    \node[hidden] (H2-m) at (2*\layersep, -3.5*\nodesep) {};

    \node[above=0.2cm of H2-1, align=center] {Layer 2};

    \node (dots-1) at (3*\layersep, \nodesep*0.5) {$\dots$};
    \node (dots-2) at (3*\layersep, -\nodesep*0.5) {$\dots$};
    \node (dots-m) at (3*\layersep, -3.5*\nodesep) {$\dots$};
    
    \node[hidden] (HL-1) at (4*\layersep, \nodesep*0.5) {};
    \node[hidden] (HL-2) at (4*\layersep, -\nodesep*0.5) {};
    \node (HL-dots) at (4*\layersep, -1.5*\nodesep) {$\vdots$};
    \node[hidden] (HL-m) at (4*\layersep, -3.5*\nodesep) {};
    
    \node[above=0.2cm of HL-1, align=center] {Layer $L$};

    \node[output] (O) at (5.5*\layersep, -1.5*\nodesep) {$y$};
    \node[above=0.5cm of O] {Output};

    \foreach \i in {I-1, I-2, I-n}
        \foreach \j in {H1-1, H1-2, H1-m}
            \draw[arrow] (\i) -- (\j);
    \path (I-1) -- (H1-m) node [midway, above=1.5cm] {$W_1$};

    \foreach \i in {H1-1, H1-2, H1-m}
        \foreach \j in {H2-1, H2-2, H2-m}
            \draw[arrow] (\i) -- (\j);
    \path (H1-1) -- (H2-m) node [midway, above=1.5cm] {$W_2$};

    \foreach \i in {H2-1, H2-2, H2-m}
        \foreach \j in {dots-1, dots-2, dots-m}
            \draw[arrow, shorten >=5pt] (\i) -- (\j);
    \path (H2-1) -- (dots-m) node [midway, above=1.5cm] {$W_3$};

    \foreach \i in {dots-1, dots-2, dots-m}
        \foreach \j in {HL-1, HL-2, HL-m}
            \draw[arrow, shorten <=5pt] (\i) -- (\j);
    \path (dots-1) -- (HL-m) node [midway, above=1.5cm] {$W_L$};

    \foreach \i in {HL-1, HL-2, HL-m}
        \draw[arrow, draw=red!50!black] (\i) -- (O);
    \node[text=red!60!black, font=\small] at (4.75*\layersep, 0) {$\ba_L$};

    \node[below=0.8cm of H1-m, font=\small] {$\sigma(W_1 \bx + \bb_1)$};
    
    \node[below=0.8cm of H2-m, font=\small] {$\sigma(W_2 \bh_1 + \bb_2)$};
    
    \node[below=0.8cm of HL-m, font=\small] {$\sigma(W_L \bh_{L-1} + \bb_L)$};
    
    \node[below=1.0cm of O, text=red!60!black, font=\small] {$y = \ba_L^T \sigma(W_L \bh_{L-1} + \bb_L)$};
\end{tikzpicture}
\caption{Schematic of a feedforward neural network with $L$ (hidden) layers.}
\label{fig:FNN-schematic}
\end{figure}
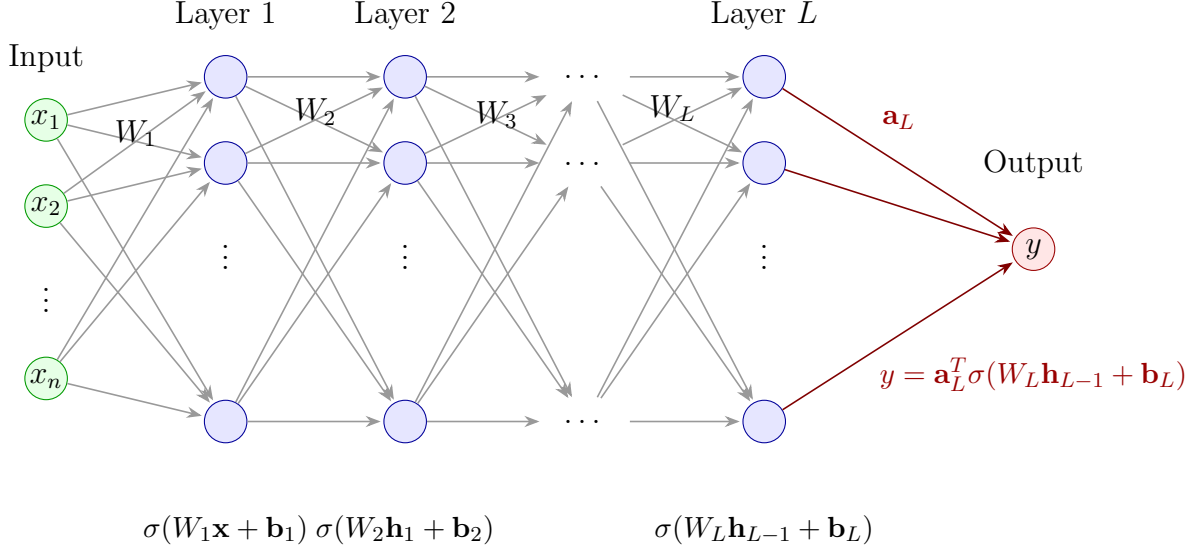

A multilayer feedforward neural network is obtained by composing several such hidden layers. More precisely, let $L\ge 1$. An \emph{$L$-layer neural network} is a mapping $F:\mathbb{R}^d\to\mathbb{R}$ defined recursively by
\[
    \mathbf{h}_0 := \mathbf{x}, \qquad
    \mathbf{h}_\ell := \sigma(W_\ell \mathbf{h}_{\ell-1} + \mathbf{b}_\ell), 
    \quad \ell=1,\ldots,L,
\]
and
\[
    F(\mathbf{x}) := \mathbf{a}_L^\top \mathbf{h}_L,
\]
where for each $\ell$, $W_\ell$ is a linear transformation matrix of appropriate dimension and $\mathbf{b}_\ell$ is a bias vector. The activation function $\sigma$ is applied coordinatewise.

The integer $L$ is called the \emph{depth} of the network (i.e., the number of \emph{hidden layers}), while the widths of the intermediate layers are determined by the dimensions of the matrices $W_\ell$. The case $L=1$ coincides with the shallow architecture in \eqref{eq:defn_single_layer}, whereas larger values of $L$ correspond to deeper networks obtained through repeated composition of affine maps and nonlinear activations. The total number of hidden neurons in a network is referred to as its \emph{size}.

\begin{remark}
    While we have used a single activation function $\sigma$ for the whole architecture, it is of course possible to allow different activations for different neurons. Thus, with such a modified definition, the map 
    \[
        \mathbf{x} \;\mapsto\; \sum_{j=1}^m a_j \,\sigma_j(\mathbf{w}_j^\top \mathbf{x} + b_j)
    \]
    would define a single-layer network.
\end{remark}

\subsection{Some basic UATs}\label{sec:basic}
In this subsection, we look at several basic UATs and prove some of them. We then state a result that provides a necessary and sufficient condition on the activation function for UATs to hold. For functions $f: \mathbb{R}^d \to \mathbb{R}$ and a measurable set
$D \subset \mathbb{R}^d$, define
\[
    \|f\|_{\infty, \, D} := \sup_{\mathbf{x} \in D} |f(\mathbf{x})|.
\]
We denote by $C(K)$ the set of continuous functions on a compact set $K$. Then $\|f\|_{\infty, \, K}$ is the uniform norm on $C(K)$.

\subsubsection{Two UATs for continuous functions}
We begin with one of the easiest (but somewhat inefficient) UATs one can prove.
\begin{theorem}\label{thm:lipschitz}
Let $f : [0,1] \to \mathbb{R}$ be a function with Lipschitz (with respect to the Euclidean norm) constant $C_f$.  
For any $\epsilon > 0$, there exists a $1$-layer neural network $F$ with width $\big\lceil \frac{C_f}{\epsilon}\big\rceil$ and the Heaviside activation function  $\sigma(x) := \ind_{[0, \infty)}(x)$, such that
\[
    \|F - f\|_{\infty, \, [0,1]} \le \epsilon.
\]
\end{theorem}

\begin{proof}
Fix $\epsilon >0$ and set $m := \big\lceil\frac{C_f}{\epsilon} \big\rceil$. Define the partition points
\[
    t_{j} := \frac{j}{m}, \quad\text{for}\,\, j= 0,1,\ldots,m-1.
\]
Introduce the coefficients
\begin{align*}
    w_0 &:= f(t_0),\\
    w_j &:= f(t_j) - f(t_{j-1}), \quad\text{for}\,\, j = 1,\dots,m-1,
\end{align*}
and define for $x\in [0,1]$,
\[
    F(x) := \sum_{j=0}^{m-1} w_j \, \ind_{[t_j, \infty)}(x) = \sum_{j=0}^{m-1} w_j \, \sigma(x-t_j).
\]

Fix $x \in [0,1]$ and let $\ell$ be the largest index $0 \le j \le m - 1$ such that $t_j \le x$.  
By construction, the function $F$ is constant on the interval $[t_\ell, t_{\ell+1})$ which contains $x$, and thus $F(x) = F(t_{\ell}) = f(t_\ell)$. Therefore
\begin{align*}
    |F(x) - f(x)|
    &=  |f(t_{\ell}) - f(x)|\\
    &\le C_f |x - t_\ell|\\
    & \le C_f |t_{\ell+1} - t_{\ell}|\\
    &= \frac{C_f}{m} \\
    &\le \epsilon,
\end{align*}
where the first inequality follows from the fact that $f$ is $C_f$-Lipschitz. This completes the proof.
\end{proof}

An obvious higher-dimensional analogue would involve functions for the form\footnote{This provides an example of a so-called $\Sigma\Pi$ network, allowing sums and products of activation functions, for which universal approximation results can be found in \cite{hornik1989multilayer}.} $\mathbf{x} \mapsto \prod_{j = 1}^d \ind_{[t_j, \infty)}(x_j)$, which are not activation functions per se, but it turns out they can be approximated in $L^1$ norm by 2-layer neural nets with ReLU activation. This gives the following result (a proof can be found in the lecture notes \cite{mjt_dlt}).
\begin{theorem}\label{thm:cont_func_mult}
For any continuous function $f : \mathbb{R}^{d} \to \mathbb{R}$ and $\epsilon > 0$, there exists a $2$-layer ReLU network $F$ such that
\[
\int_{[0,1]^d} |F(\mathbf{x})-f(\mathbf{x})|\,d\mathbf{x} \le \epsilon .
\]
\end{theorem}
\begin{remark}
    In the proof of Theorem~\ref{thm:cont_func_mult}, one first chooses a $\delta > 0$ such that if any $\mathbf{x}, \mathbf{x}' \in [0, 1]^d$ are coordinatewise $\delta$-apart, then $|f(\mathbf{x}) - f(\mathbf{x}')| \le \epsilon / 2$, and then partitions $[0, 1]^d$ into cells of side length $\delta$. The final ReLU network has of the order of $(1 / \delta)^d$ neurons.
\end{remark}

\subsubsection{Cybenko's theorems}
We now present a few results from \cite{cybenko1989approximation}. His results are amongst the the first rigorous UATs for feedforward neural networks with general sigmoidal activation functions. The first result we shall look at shows that networks with a single hidden layer and a continuous sigmoidal activation function are dense in $C(K)$, where $K$ is any compact subset of $\bbR^d$. The proof employs a functional analytic argument using the Hahn--Banach theorem and the Riesz representation theorem.

\begin{definition}[Sigmoidal function]
A bounded measurable function $\sigma : \mathbb{R} \to \mathbb{R}$ is called \emph{sigmoidal} if
\[
\lim_{t \to -\infty} \sigma(t) = 0
\quad \text{and} \quad
\lim_{t \to +\infty} \sigma(t) = 1.
\]
\end{definition}
\begin{theorem} [\cite{cybenko1989approximation}, Theorem 2]\label{thm:cybenko-1}
Let $\sigma: \bbR \to \bbR$ be a continuous sigmoidal function and $f : [0,1]^d \to \bbR$ be continuous. Then given any $\epsilon >0$, there exists a $1$-layer neural network $F$ such that
\[
    \|F - f\|_{\infty, \, [0,1]^d}\le \epsilon.
\]
 \end{theorem}

\begin{proof}
    Let $\mathcal{H}$ denote the collection of all functions on $[0,1]^d$ of the form
    \[
        F(\mathbf{x}) = \sum_{k=1}^m a_k \, \sigma(\mathbf{w}_k^\top \mathbf{x} + b_k),
    \]
    where $m \in \mathbb{N}$, $a_k,b_k \in \mathbb{R}$, and
    $\mathbf{w}_k \in \mathbb{R}^d$.
    We show that $\mathcal{H}$ is dense in $C([0,1]^d)$ with respect to the
    uniform norm, which would prove the theorem.
    
    Assume, for contradiction, that $\mathcal{H}$ is not dense in
    $C([0,1]^d)$.
    Then its closure $\overline{\mathcal{H}}$ is a proper closed linear subspace
    of $C([0,1]^d)$. By the Hahn--Banach separation theorem, there exists a \emph{non-zero} bounded linear
    functional $\Lambda : C([0,1]^d) \to \mathbb{R}$ such that
    \[
        \Lambda(F) = 0 \quad \text{for all } F \in \overline{\mathcal{H}}.
    \]
    By the Riesz representation theorem, there exists a finite signed Borel measure $\mu$ on $[0,1]^d$, not identically zero, such that
    \[
        \Lambda(g) = \int_{[0,1]^d} g(\mathbf{x}) \, d\mu(\mathbf{x})
        \quad \text{for all } g \in C([0,1]^d).
    \]
    Therefore,
    \begin{equation}\label{eq:annihilation}
        \int_{[0,1]^d} \sigma(\mathbf{w}^\top \mathbf{x} + b)\, d\mu(\mathbf{x}) = 0
        \quad \text{for all } \mathbf{w} \in \mathbb{R}^d \text{ and } b \in \mathbb{R}.
    \end{equation}
    Since $\sigma$ is sigmoidal, for each fixed
    $\mathbf{w} \in \bbR^d$, and $b, c \in \bbR$, the function
    $\bx \mapsto \sigma(\lambda(\mathbf{w}^\top \mathbf{x} + b) + c)$ converges pointwise, as $\lambda \to \infty$, to the function 
    \[
        \ind_{H_{\mathbf{w}, b}}(\mathbf{x}) + \sigma(c) \ind_{\partial H_{\mathbf{w}, b}}(\mathbf{x}),
    \]
    where $H_{\mathbf{w}, b} := \{\mathbf{x} : \mathbf{w}^\top \mathbf{x} + b > 0\}$ is the open half-space determined by $\mathbf{w}$ and $b$, the half-plane $\partial H_{\mathbf{w}, b} = \{\mathbf{x} : \mathbf{w}^\top \mathbf{x} + b = 0\}$ being its \emph{boundary}. By the dominated convergence theorem, \eqref{eq:annihilation} implies that
    \begin{equation}\label{eq:mu-halfspace-identity}
        \mu(H_{\mathbf{w}, b}) + \sigma(c) \mu(\partial H_{\mathbf{w}, b}) = 0
        \quad \text{for all } \mathbf{w} \in \bbR^d, \text{ and } b, c \in \bbR.
    \end{equation}
    Sending $c \to \pm\infty$, we see that $\mu$ assigns zero measure to every open half-space $H_{\mathbf{w}, b}$ and every closed half-space $\overline{H}_{\mathbf{w}, b} := H_{\mathbf{w}, b} \cup \partial H_{\mathbf{w}, b} = \{\mathbf{x} : \mathbf{w}^\top \mathbf{x} + b \ge 0\}$. We will show that this implies $\mu = 0$. (This is not immediate since $\mu$ is a signed measure.)

    For a fixed $\mathbf{w} \in \bbR^d$, define a bounded linear functional $\tilde{\Lambda}$ on $L^{\infty}(\bbR)$ as follows:
    \[
        \tilde{\Lambda}(h) := \int h(\mathbf{w}^\top \mathbf{x}) d\mu(\mathbf{x}).
    \]
    Taking $h = \ind_{[b, \infty)}$ or $\ind_{(b, \infty)}$, we see that
    \begin{align*}
        \tilde{\Lambda}(\ind_{[b,\infty)}) = \mu(\overline{H}_{\mathbf{w}, -b}) = 0, \\
        \tilde{\Lambda}(\ind_{(b,\infty)}) = \mu(H_{\mathbf{w}, -b}) = 0.
    \end{align*}
    It follows that $\tilde{\Lambda}(\ind_A) = 0$ for any interval $A$, and a fortiori, $\tilde{\Lambda}(h) = 0$ whenever $h$ is a finite linear combination of indicator functions of intervals. As such functions are dense in $L^{\infty}(\bbR)$, we may conclude that $\tilde{\Lambda} = 0$.

    In particular, this gives that
    \[
        \int e^{\iota \mathbf{w}^\top \mathbf{x}} d\mu(\mathbf{x}) = \tilde{\Lambda}(\cos) + \iota \cdot \tilde{\Lambda}(\sin) = 0, \qquad \iota:= \sqrt{-1},
    \]
    so that the Fourier transform of $\mu$ is $0$, implying that $\mu = 0$, which is the desired contradiction. This completes the proof.
\end{proof}

\cite{cybenko1989approximation} also proves a similar result for simple functions (i.e., measurable functions with finite range), which model decision functions.
\begin{theorem}[\cite{cybenko1989approximation}, Theorem 3]
    Suppose $\sigma: \bbR \to \bbR$ is a continuous sigmoidal function and $f:[0,1]^d \to \bbR$ is a simple function. Then for any $\epsilon >0$, there exists a measurable set $D \subset [0,1]^d$ with $\mathrm{Vol}(D) \ge 1- \epsilon$ and a $1$-layer neural network $F:[0,1]^d \to \bbR$ such that
    \[
        \|F - f\|_{\infty, D} \le \epsilon.
    \]
\end{theorem}
The proof is easy, since by Lusin's theorem, one can find a continuous function $h$ and a measurable set $D$, with $\mathrm{Vol}(D) \ge 1 - \epsilon$, such that $f = h$ on $D$. Then one may apply Theorem~\ref{thm:cybenko-1}. 

\begin{remark}
    \cite{cybenko1989approximation} proves an analogue of Theorem~\ref{thm:cybenko-1} for $L^1([0, 1]^d)$, where it suffices to take $\sigma$ to be any sigmoidal function.
\end{remark}

\subsubsection{Necessary and sufficient condition on \texorpdfstring{$\sigma$}{sigma} for UAT}
We end this subsection with a theorem from \cite{leshno1993multilayer} that provides a necessary and sufficient condition on the activation function $\sigma$ for universal approximation.

\begin{theorem}[\cite{leshno1993multilayer}, Theorem 1]
Let $\sigma : \mathbb{R} \to \mathbb{R}$ be continuous. Then the set of $1$-layer feedforward networks is dense in $C([0,1]^d)$ with respect to the uniform norm if and only if $\sigma$ is not a polynomial.  
\end{theorem}
We refer the reader to the survey \cite{pinkus1999approximation} for an exposition of the proof. In fact, \cite{pinkus1999approximation} attributes this result to \cite{schwartz1944certaines}.

\subsection{Quantitative universal approximation: An early result of Barron}\label{sec:barron}
What class of functions do single-layer neural networks approximate well? \cite{barron1993universal} identified the right class of functions which is defined in terms of the decay of their Fourier transforms. Barron showed that such functions admit approximation by $1$-layer feedforward networks with error rates that decay at least as fast as one over the square-root of the number of neurons. Here we mention one of the simplest results from \cite{barron1993universal}.

\begin{definition}[Barron norm]
Let $f : \mathbb{R}^d \to \mathbb{R}$ be a function admitting a Fourier representation
\[
    f(\mathbf{x}) = \int_{\mathbb{R}^d} e^{\iota \mathbf{w}^\top \mathbf{x}} \, \tilde{f}(\mathbf{w}) \, d\mathbf{w},
\]
where $\tilde{f} : \mathbb{R}^d \to \mathbb{C}$ is a complex-valued function. Define the \emph{Barron norm} of $f$ by
\[
    \|f\|_B := \int_{\mathbb{R}^d} \|\mathbf{w}\| \, |\tilde{f}(\mathbf{w})| \, d\mathbf{w},
\]
where $\|\mathbf{w}\| = (\mathbf{w}^\top \mathbf{w})^{1/2}$ is the Euclidean norm.
\end{definition}
For $r > 0$ we denote by
\[
    B_r := \{ \mathbf{x} \in \mathbb{R}^d : \|\mathbf{x}\| \le r \}
\]
the closed Euclidean ball of radius $r$ centered at the origin.
\begin{theorem}[\cite{barron1993universal}, Proposition~1]
Let $f : \mathbb{R}^d \to \mathbb{R}$ be a function with finite Barron norm, and let $\mu$ be any probability measure supported on $B_r$ for some $r > 0$. Let $\sigma$ be a sigmoidal activation. Then for any integer $m \ge 1$, there exists a $1$-layer feedforward neural network $F_m$, with $m + 1$ neurons, such that
\[
    \int_{B_r} \bigl( F_m(\mathbf{x}) - f(\mathbf{x}) \bigr)^2 \, \mu(d\mathbf{x})
    \le \frac{4 \, r^2 \, \|f\|_B^2}{m}.
\]
\end{theorem}

Barron's work was later generalised significantly by \cite{makovoz1996random}, who obtained the sharper rate of $\mathcal{O}(m^{-\frac{d + 1}{d}})$, which is essentially optimal. See \cite{pinkus1999approximation} for a discussion. Modern works studying approximation by neural networks in Barron-type function classes include \cite{klusowski2016risk, e2021barron, siegel2022high}.

\subsection{The benefit of depth}\label{sec:depth}
While early results such as those of Cybenko and Barron focused primarily on shallow $1$-layer networks, a growing body of work has highlighted the expressivity advantages of depth in neural architectures. Deeper networks can represent certain classes of functions exponentially more efficiently than shallow ones, requiring far fewer neurons to achieve the same approximation accuracy.

For example, \cite{telgarsky2015representation} gave a family of classification problems indexed by a positive integer $k$ such that any shallow ReLU network with fewer than exponentially many nodes in $k$ incurs a constant classification error, whereas a deep feedforward ReLU network with at most $2$ nodes in each of $2k$ layers achieves zero error. \cite{telgarsky2016benefits} proved a more general result of this kind, which we shall look at shortly. The almost concurrent work of \cite{eldan2016power} exhibited functions expressible by 2-layer networks of polynomial width in the input dimension, that are not approximable by 1-layer ReLU networks unless their widths are exponential in the input dimension, thus establishing explicit separation. In a similar vein, the slightly earlier work of \cite{montufar2014number} analyzed the number of linear regions induced by piecewise-linear activations such as ReLU, demonstrating that deeper networks partition the input space more finely into convex polytopes and thus can express more complex functions with a fixed total number of neurons. Another notable work of this type is \cite{bianchini2014complexity}, which compares deep and shallow networks in terms of a certain topological complexity measure of the functions they represent. These results, together with follow-up works such as \cite{poole2016exponential, raghu2017expressive}, provide solid theoretical justification for the use of deep architectures, showing that depth alone is a powerful mechanism for increasing representational capacity without necessarily increasing width.

\subsubsection{Telgarsky's construction}
\cite{telgarsky2016benefits} demonstrated a quantitative separation between deep and shallow neural networks by constructing functions that can be computed by relatively small deep ReLU networks but cannot be approximated by any shallow network unless it has exponentially many nodes.

Telgarsky's result applies to a broad class of activation functions called \textit{semi-algebraic gates}.
\begin{definition}[Semi-algebraic gates]
A function $\sigma : \mathbb{R}^d \to \mathbb{R}$ is called \textit{$(t,\alpha,\beta)$-semi-algebraic} if there exist polynomials  $q_i : \mathbb{R}^d \to \mathbb{R}$, $i= 1, \ldots, t$, each of degree at most $\alpha$, and triplets $(p_i, U_i, V_i)_{i=1}^m$, where each $p_i:\bbR^d \to \bbR$ is a polynomial of degree at most $\beta$, and each $U_i, V_i \subset \{1,\ldots, t\}$ with $U_i \cap V_i =\emptyset$, such that
\[
    \sigma(\mathbf{x}) = \sum_{j=1}^{m} p_j(\mathbf{x}) \bigg( \prod_{i \in U_j} \ind\big( q_i(\mathbf{x}) \ge 0 \big) \bigg) \bigg( \prod_{i \in V_j} \ind\big( q_i(\mathbf{x}) < 0 \big) \bigg).
\]
\end{definition}
In simpler terms, a semi-algebraic gate partitions the input space into finitely many regions using polynomial inequalities, and on each region behaves like a polynomial of bounded degree. It is a piecewise polynomial function with controlled algebraic complexity. Examples of semi-algebraic gates include many nonlinearities used in practice, such as ReLU, Leaky-ReLU, max gates, indicator functions of half-spaces, piecewise polynomial activations, etc.

Telgarsky's construction is explicit as well as elementary. It essentially builds highly oscillatory functions through repeated composition of a simple \textit{tent function}. The oscillations of a function $f : \bbR \to \bbR$ are formally counted by its \emph{crossing number}, the smallest number of intervals in a partition of $\bbR$ into intervals on which the function $x \mapsto \ind(f(x) \ge 1/2)$ is constant. 
\begin{definition}[\textit{Tent function}]
Define $\tau : \bbR \to [0,1]$ by
\[
    \tau(x) := \begin{cases}
        2x, & \text{if } x \in [0, 1/2], \\
        2(1 - x), & \text{if } x \in (1/2, 1], \\
        0, & \text{otherwise.}
    \end{cases}
\]
\end{definition}
For $k \ge 1$, let $\tau^{[k]}$ denote the $k$-fold composition of $\tau$ with itself, i.e., 
\[
    \tau^{[k]}:= \underbrace{\tau \circ \tau \circ \dots \circ \tau}_{k \,\text{times}}.
\]
The function $\tau$ increases linearly from 0 to 1 and then decreases linearly back to 0, forming one peak (and has crossing number $3$). Each composition of $\tau$ with itself doubles the number of peaks; so, $\tau^{[k]}$ has exactly $2^k$ peaks (and crossing number $2^k + 1$). This exponential growth under composition is at the heart of Telgarsky's depth separation argument. We now present Telgarsky's result in a simplified form.
\begin{theorem}[\cite{telgarsky2016benefits}, Theorem~1.1]\label{thm:telgarsky}
Fix integers $k \ge 1$ and $d \ge 1$. Let $\cC$ be the collection of all feedforward neural networks $F : \bbR^d \to \bbR$ with $(t, \alpha, \beta)$-semi-algebraic activations, at most $k$ layers and at most $2^k / (t\alpha\beta)$ nodes. Then, there exists a ReLU neural network $f : \mathbb{R}^d \to \mathbb{R}$ with $2k^3 + 8$ layers, $3k^3 + 12$ total nodes, $d + 4$ distinct parameters, such that 
\[
    \inf_{F \in \cC} \int_{[0,1]^d} |F(x) - f(x)| \, dx \ge \frac{1}{64}.
\]
\end{theorem}
The separating function $f$ in Theorem~\ref{thm:telgarsky} is essentially the function\footnote{An observant reader might complain that $\tau^{[k^3 + 4]}$ is univariate! Telgarsky's $f$ in fact only depends on a single coordinate of $\mathbf{x} \in [0, 1]^d$.} $\tau^{[k^3 + 4]}$. The function $\tau$ can be implemented exactly by a shallow ReLU network (indeed, one may check that if $\sigma_{\mathrm{R}}(x)$ denotes the ReLU function, then $\tau(x) = \sigma_{\mathrm{R}}(2\sigma_{\mathrm{R}}(x) - \sigma_{\mathrm{R}}(x - 1/2))$). By stacking (composing) this network $k^3 + 4$ times, one obtains a deep ReLU network computing the function $\tau^{[k^3 + 4]}$ (and hence $f$) exactly, which has $2^{k^3 + 4} + 1$ oscillations. On the other hand, a key result in \cite{telgarsky2016benefits} (see Lemma~3.2 therein) implies that any function computed by a $(t, \alpha, \beta)$-semi-algebraic network with at most $k$ layers and size at most $2^k/(t\alpha \beta)$ has strictly fewer oscillations. This then leads to the stated $L^1$-inapproximability result.
\begin{remark}
    In fact, Telgarsky's lower bound is much more general, in that it holds even when one includes certain convolutional networks or boosted decision trees in the class $\cC$.
\end{remark}
Telgarsky's theorem mirrors classical depth hierarchy results in circuit complexity: shallow circuits cannot simulate deeper circuits without exponential blow-up in size. Likewise, shallow neural networks cannot efficiently simulate certain deep networks.

A particularly striking aspect is parameter efficiency: the deep network uses only $O(k^3)$ nodes and $d + 4$ distinct parameters, yet any shallow competitor requires exponential size. An intuitive explanation of this expressivity of deep networks would be as follows. Many real-world learning problems (e.g., vision, language, speech) exhibit hierarchical structure. Deep networks naturally mirror such structure by building increasingly abstract representations layer by layer. A shallow network must flatten this hierarchy into a single transformation, leading to massive parameter growth.

Another important aspect of this result is that the separation between deep and shallow networks already appears in one dimension. Thus the expressivity advantage does not arise from high-dimensional geometry but from hierarchical composition. Depth enables repeated functional iteration; repeated iteration yields exponential oscillatory complexity; shallow architectures cannot simulate this phenomenon without exponential resources.

\subsubsection{Minimum width for universal approximation}
We have seen that depth increases the expressive power of neural networks. Suppose we limit the width but allow unlimited depth. How small can the depth be for universal approximation to still hold? There have been a number of important works that essentially characterise the minimum width required for universal approximation \citep{lu2017expressive, hanin2017approximating, johnson2018deep, kidger2020universal, park2021minimum, cai2023achieve, li2023minimum, kim2024minimum, rochau2024new, hwang2025optimal}. The minimum width required depends on the input and the output dimensions, as well as the function class considered and the type of activation function used. 

We now present a sample result from \cite{li2023minimum}, that
establishes a sharp characterisation of the minimum width required for fully connected feedforward neural networks with Leaky-ReLU activation to achieve uniform universal approximation on
compact subsets of $\mathbb{R}^d$. The Leaky-ReLU activation is defined as
\[
\sigma_{\alpha}(x) := \begin{cases}
    x & \text{if } x > 0, \\
    \alpha x & \text{if } x \le 0,
\end{cases}
\]
where $\alpha \in (0,\infty)\setminus \{1\}$ is a fixed parameter.

Fix $\alpha \in (0, \infty)\setminus \{1\}$. Let $\mathcal{N}_{m}$
denote the class of neural networks mapping $\bbR^{d_x}$ to $\bbR^{d_y}$, with activation $\sigma_\alpha$, and having
arbitrary depth and fixed width $m$.

\begin{theorem}[\cite{li2023minimum}, Theorem~2.2]
Let $K \subset \mathbb{R}^{d_x}$ be a compact set with nonempty interior.
Then the class $\mathcal{N}_{m}$ is dense in $C(K; \bbR^{d_y})$ with respect to the uniform norm if and only if
\[
    m \ge \max\{d_x + 1, d_y\} + \ind(d_y = d_x + 1).
\]
\end{theorem}
Thus $d + 1$ is the minimal width guaranteeing uniform universal
approximation by scalar-valued Leaky-ReLU networks on $\bbR^{d}$. Incidentally, \cite{hanin2017approximating} showed that this is also the case for ReLU networks.

\subsubsection{A quantitative UAT featuring the role of width and depth}
Many recent works have studied quantitative approximation properties of deep ReLU networks. Some notable ones are \cite{yarotsky2017error, yarotsky2018optimal, petersen2018optimal, kohler2021rate, daubechies2022nonlinear, siegel2023optimal}. We now consider a quantitative UAT from \cite{kohler2021rate} that quantifies how approximation accuracy depends on network depth and width, and makes explicit the trade-off between these architectural parameters. We need some notation first.

\begin{definition}[Multi-index notation]
Let $\mathbf{k} = (k_1,\dots,k_d) \in (\mathbb{N} \cup \{0\})^d$. Its \emph{order} $|\mathbf{k}|$ is defined as 
\[
    |\mathbf{k}| := k_1 + \cdots + k_d.
\]
For a sufficiently smooth function $f : \mathbb{R}^d \to \mathbb{R}$, the notation
$\partial^{\mathbf{k}} f$ denotes the mixed partial derivative of order
$\mathbf{k}$, defined as
\[
    \partial^\mathbf{k} f(\bx) := \frac{\partial^{|\mathbf{k}|}f}{\partial {x_1}^{k_1} \ldots \partial {x_d}^{k_d}} (\bx),
\]
where $\bx := (x_1, \ldots, x_d)$.
\end{definition}

\begin{definition}[$(p,C)$-smooth functions] \label{def:p_C_smooth}
Let $C > 0$, and $p = q + s$ with $q \in \mathbb{N} \cup \{0\}$ and $s \in (0,1]$.
A function $f : \mathbb{R}^d \to \mathbb{R}$ is called $(p,C)$-smooth if
\begin{enumerate}
  \item[(i)] all partial derivatives $\partial^{\mathbf{k}} f$ with
  $|\mathbf{k}|= q$ exist, and 
  \item[(ii)] for every multi-index $\mathbf{k}$ with $|\mathbf{k}| = q$,
  the derivative $\partial^{\mathbf{k}} f$ is $s$-H\"older continuous with constant $C$, i.e.,
  \[
  \big|\partial^{\mathbf{k}} f(\mathbf{x}) -
  \partial^{\mathbf{k}} f(\mathbf{y})\big|
  \le C \|\mathbf{x} - \mathbf{y}\|^s
  \quad \text{for all } \mathbf{x},\mathbf{y} \in \mathbb{R}^d.
  \]
\end{enumerate}
\end{definition}

\begin{definition}[$C^q$-norm]
Let $D \subset \mathbb{R}^d$ be an open set and $f : D \to \mathbb{R}$ be $q$ times continuously
differentiable. Then for any compact $K \subset D$, define
\[
    \|f\|_{C^q(K)} := \max_{|\mathbf{k}| \le q} \|\partial^{\mathbf{k}}f\|_{\infty, K}.
\]
\end{definition}

\begin{definition}[Network class $\mathcal{F}(L,m)$]
For $L, m \in \mathbb{N}$, we denote by $\mathcal{F}(L,m)$ the class of all feedforward ReLU
neural networks of depth $L$ such that each hidden layer contains at most
$m$ neurons.
\end{definition}

\begin{theorem}[\cite{kohler2021rate}, Theorem 2]\label{thm:kohler-langer}
Let $d \in \mathbb{N}$ and $f : \mathbb{R}^d \to \mathbb{R}$ be
$(p,C)$-smooth as in Definition~\ref{def:p_C_smooth}. Let $a \ge 1$ and $M\ge 2$ be a sufficiently large integer satisfying
\[
    M^{2p} \ge c \cdot \max\bigl\{a,\|f\|_{C^q([-a,a]^d)}\bigr\}^{4(q+1)}
\]
for some sufficiently large universal constant $c \ge 1$.

\smallskip
\noindent
\textbf{(a) Wide-network approximation.}
Suppose $L,m \in \mathbb{N}$ satisfy
\begin{enumerate}
\item [(i)] $L \ge 5 + \left\lceil \log_4(M^{2p}) \right\rceil \left( \left\lceil \log_2(\max\{q,d\}+1) \right\rceil + 1 \right)$,
\item[(ii)] $m \ge 2^d \cdot 64 \cdot \binom{d+q}{d} \cdot d^2 \cdot (q+1) \cdot M^d$.
\end{enumerate}
Then there exists a neural network
$F_{\mathrm{wide}} \in \mathcal{F}(L,m)$ such that
\begin{equation}\label{eq:kohler-langer-wide}
    \|f - F_{\mathrm{wide}}\|_{\infty,[-a,a]^d} \le c' \cdot \max\bigl\{a,\|f\|_{C^q([-a,a]^d)}\bigr\}^{4(q+1)} \cdot M^{-2p}.
\end{equation}

\smallskip
\noindent
\textbf{(b) Deep-network approximation.}
Suppose $L,m \in \mathbb{N}$ satisfy
\begin{enumerate}
\item [(i)] $L \ge 5 M^d +  \left\lceil \log_4\!\Big( M^{2p + 4d(q+1)} \, e^{4(q+1)(M^d - 1)} \Big) \right\rceil \left\lceil \log_2(\max\{q,d\}+1) \right\rceil + \left\lceil \log_4(M^{2p}) \right\rceil$,
\item [(ii)] $m \ge 132 \cdot 2^d \cdot \lceil e^d \rceil \cdot \binom{d+q}{d} \cdot \max\{q+1,d^2\}$.
\end{enumerate}
Then there exists a neural network
$F_{\mathrm{deep}} \in \mathcal{F}(L,m)$
such that inequality \eqref{eq:kohler-langer-wide} holds with
$F_{\mathrm{wide}}$ replaced by $F_{\mathrm{deep}}$.
\end{theorem}

Theorem~\ref{thm:kohler-langer} shows that smooth functions can be approximated with the same accuracy either by wide networks of moderate depth or by deep networks with
controlled width. It provides a nice quantitative description of the depth--width trade-off in neural network approximation.

\section{Kolmogorov--Arnold Networks}\label{sec:kan}
\emph{Kolmogorov--Arnold Networks} (KANs) are a recently proposed \citep{liu2025kan} neural network architecture explicitly inspired by the remarkable \emph{Kolmogorov--Arnold representation theorem}, which states that any multivariate continuous function can be written as a finite superposition of continuous univariate functions and addition operations \citep{kolmogorov1957representations}. Unlike feedforward neural networks that employ fixed activation functions at neurons, KANs replace each linear weight parameter with a learnable univariate activation function, often modelled using basis functions such as B-splines. In this way, KANs directly parameterise the functional components suggested by the Kolmogorov--Arnold theorem within a trainable architecture \citep{liu2025kan}.

This structural innovation has consequences for approximation theory. Compared with classical UATs for FNNs, which guarantee existence of approximating networks in principle, the Kolmogorov--Arnold approach embeds the representation theorem directly into the architecture and yields constructive approximation guarantees that reflect both function smoothness and architectural design. This makes KANs an intriguing bridge between classical approximation theory and modern deep learning practice. We shall now review the Kolmogorov--Arnold representation theorem and the basic architecture of a KAN. We shall also look at a basic approximation theorem for KANs due to \cite{liu2025kan}.

\subsection{Hilbert's 13-th problem and the Kolmogorov--Arnold representation theorem}\label{sec:kn-repr}
The Kolmogorov--Arnold representation theorem originates in the series of works \cite{kolmogorov1956representations, arnol1957functions, arnol1959representation, kolmogorov1957representations}, and resolves a variant of Hilbert's 13-th Problem (one of the twenty three problems posed at the turn of the twentieth century by the great mathematician David Hilbert \citep{Hilbert1902}). Hilbert had conjectured that solutions of general seventh-degree algebraic equations require genuinely multivariate functions in the coefficients, which implies that there exist continuous functions of three variables that can not be expressed in terms of composition and addition of functions of two variables. Building on Kolmogorov's work \citep{kolmogorov1956representations}, \cite{arnol1957functions} disproved this conjecture. Kolmogorov followed up \citep{kolmogorov1957representations} by establishing the remarkable result that every continuous function $f : [0,1]^d \to \mathbb{R}$ can be represented as a finite superposition of continuous univariate functions and addition. In the literature, this result is also often called Kolmogorov's superposition theorem (see \cite{morris2021hilbert} for an historical account).

\begin{theorem}[Kolmogorov--Arnold Representation Theorem]
Let $f \in C([0,1]^d)$. Then there exist continuous univariate functions
\[
    \varphi_{q,p} : \mathbb{R} \to \mathbb{R}, \qquad \psi_q : \mathbb{R} \to \mathbb{R},
\]
for $q = 1,\dots,2d+1$ and $p = 1,\dots,d$, such that for all $\mathbf{x} = (x_1,\dots,x_d) \in [0,1]^d$,
\begin{equation}\label{eq:ka-representation}
    f(\mathbf{x}) = \sum_{q=1}^{2d+1} \psi_q\!\left( \sum_{p=1}^{d} \varphi_{q,p}(x_p) \right).
\end{equation}
Moreover, the inner functions $\varphi_{q,p}$ can be chosen independently of $f$,
while the outer functions $\psi_q$ depend on $f$.
\end{theorem}
This theorem is rather surprising at first sight because it demonstrates that multivariate continuous functions do not require intrinsically multivariate building blocks: composition and addition of one-dimensional functions suffice. Although the representation is highly non-unique and the inner functions are not smooth in general, the theorem remains a  foundational result in nonlinear approximation theory and provides the conceptual underpinning for architectures inspired by superposition principles, such as KANs.

Notably, long before the advent of KANs, the Kolmogorov--Arnold theorem was widely cited in the neural networks literature, in regards to both approximation theory and neural network design \citep{hecht1987kolmogorov, girosi1989representation, kuurkova1991kolmogorov, kuurkova1992kolmogorov, pinkus1999approximation, igelnik2003kolmogorov}. 

\subsection{The KAN architecture}
\label{subsec:kan_architecture}
We begin by defining the basic building block of a KAN, namely, a KAN layer.
\begin{definition}[KAN layers]
Let \( d_{\mathrm{in}}, d_{\mathrm{out}} \in \mathbb{N} \).
A \emph{Kolmogorov--Arnold Network (KAN) layer} $\Phi$
is a collection of univariate functions
\[
    \Phi = \{ \varphi_{q,p} : \mathbb{R} \to \mathbb{R} \mid p = 1,\dots,d_{\mathrm{in}},\; q = 1,\dots,d_{\mathrm{out}} \}.
\]
\end{definition}

\begin{definition}[The action of a KAN layer]
Given an input vector
\(
\mathbf{x} = (x_1,\dots,x_{d_{\mathrm{in}}}) \in \mathbb{R}^{d_{\mathrm{in}}},
\)
a KAN layer $\Phi$ applied to $\mathbf{x}$ produces an output $\mathbf{y} = (y_1,\dots,y_{d_{\mathrm{out}}}) \in \mathbb{R}^{d_{\mathrm{out}}}$ defined componentwise by
\begin{equation}\label{eq:kan_layer}
    y_q = \sum_{p=1}^{d_{\mathrm{in}}} \varphi_{q,p}(x_p), \qquad q = 1,\dots,d_{\mathrm{out}}.
\end{equation}
We write this succinctly as $\mathbf{y} = \Phi(\mathbf{x})$, identifying the KAN layer $\Phi$ (a collection of univariate functions) with the map $\Phi : \bbR^{d_{\mathrm{in}}} \to \bbR^{d_{\mathrm{out}}}$ defined by \eqref{eq:kan_layer}.
\end{definition}

\begin{figure}[!t]
    \centering
    \begin{tikzpicture}[
    scale=0.9,
    >=stealth,
    neuron/.style={circle, draw, minimum size=1cm},
    input/.style={neuron, fill=green!20},
    hidden/.style={neuron, fill=blue!20},
    output/.style={neuron, fill=red!20},
    edge/.style={->, line width=0.6pt}
]

\node[input]  (x1) at (0,  0.0) {$x_1$};
\node[input]  (x2) at (0, -1.6) {$x_2$};
\node at (0,-3.0) {$\vdots$};
\node[input]  (xn) at (0, -4.4) {$x_{d_0}$};
\node[above] at (0, 0.85) {Input};

\node[hidden] (h11) at (3.5,  0.0) {$+$};
\node[hidden] (h12) at (3.5, -1.6) {$+$};
\node at (3.5,-3.0) {$\vdots$};
\node[hidden] (h1n) at (3.5, -4.4) {$+$};
\node[above] at (3.5, 0.85) {Layer $1$};

\node[hidden] (h21) at (7.0,  0.0) {$+$};
\node[hidden] (h22) at (7.0, -1.6) {$+$};
\node at (7.0,-3.0) {$\vdots$};
\node[hidden] (h2n) at (7.0, -4.4) {$+$};
\node[above] at (7.0, 0.85) {Layer $2$};

\node (e1) at (10.2,  0.0) {$\cdots$};
\node (e2) at (10.2, -1.6) {$\cdots$};
\node (e3) at (10.2, -3.0) {$\ddots$};
\node (en) at (10.2, -4.4) {$\cdots$};

\node[hidden] (hL1) at (13.5,  0.0) {$+$};
\node[hidden] (hL2) at (13.5, -1.6) {$+$};
\node at (13.5,-3.0) {$\vdots$};
\node[hidden] (hLn) at (13.5, -4.4) {$+$};
\node[above] at (13.5, 0.85) {Layer $L$};

\node[output] (o1) at (16,  0.0) {$y_1$};
\node[output] (o2) at (16, -1.6) {$y_2$};
\node at (16,-3.0) {$\vdots$};
\node[output] (on) at (16, -4.4) {$y_{d_L}$};
\node[above] at (16, 0.85) {Output};

\draw[edge] (hL1) -- (o1);
\draw[edge] (hL2) -- (o2);
\draw[edge] (hLn) -- (on);

\draw[edge] (x1) -- (h11) node[midway, above, font=\small] {$\varphi^{(0)}_{1,1}$};
\draw[edge] (x2) -- (h11);
\draw[edge] (xn) -- (h11);
\draw[edge] (x1) -- (h12);
\draw[edge] (x2) -- (h12) node[midway, above, font=\small] {$\varphi^{(0)}_{2,2}$};
\draw[edge] (xn) -- (h12);
\draw[edge] (x1) -- (h1n);
\draw[edge] (x2) -- (h1n);
\draw[edge] (xn) -- (h1n) node[midway, above, font=\small] {$\varphi^{(0)}_{d_1,d_0}$};

\draw[edge] (h11) -- (h21) node[midway, above, font=\small] {$\varphi^{(1)}_{1,1}$};
\draw[edge] (h12) -- (h21);
\draw[edge] (h1n) -- (h21);

\draw[edge] (h11) -- (h22);
\draw[edge] (h12) -- (h22) node[midway, above, font=\small] {$\varphi^{(1)}_{2,2}$};
\draw[edge] (h1n) -- (h22);

\draw[edge] (h11) -- (h2n);
\draw[edge] (h12) -- (h2n);
\draw[edge] (h1n) -- (h2n) node[midway, above, font=\small] {$\varphi^{(1)}_{d_2,d_1}$};

\draw[edge] (h21) -- (e1) node[midway, above, font=\small] {$\varphi^{(2)}_{1,1}$};
\draw[edge] (h22) -- (e1);
\draw[edge] (h2n) -- (e1);

\draw[edge] (h21) -- (e2);
\draw[edge] (h22) -- (e2) node[midway, above, font=\small] {$\varphi^{(2)}_{2,2}$};
\draw[edge] (h2n) -- (e2);

\draw[edge] (h21) -- (en);
\draw[edge] (h22) -- (en);
\draw[edge] (h2n) -- (en) node[midway, above, font=\small] {$\varphi^{(2)}_{d_3,d_2}$};

\draw[edge] (e1) -- (hL1) node[midway, above, font=\small] {$\varphi^{(L - 1)}_{1,1}$};
\draw[edge] (e2) -- (hL1);
\draw[edge] (en) -- (hL1);

\draw[edge] (e1) -- (hL2);
\draw[edge] (e2) -- (hL2) node[midway, above, font=\small] {$\varphi^{(L - 1)}_{2,2}$};
\draw[edge] (en) -- (hL2);

\draw[edge] (e1) -- (hLn);
\draw[edge] (e2) -- (hLn);
\draw[edge] (en) -- (hLn) node[midway, above, font=\small] {$\varphi^{(L - 1)}_{d_{L - 1}, d_L}$};


\end{tikzpicture}
    \caption{Schematic of a KAN with $L$ layers: each edge here represents a learnable (typically nonlinear) univariate function $\varphi^{(l)}_{j,i}$, and each blue node aggregates by summation. Thus, unlike FNNs, nonlinearities are associated with edges rather than nodes.}
    \label{fig:kan-schematic}
\end{figure}
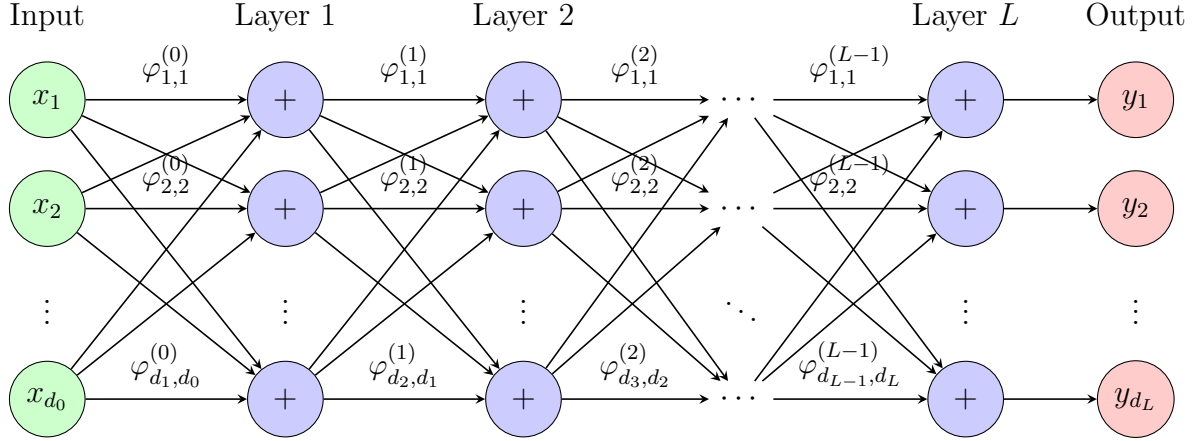

\begin{definition}[The KAN architecture]
A KAN of depth $L$ and shape $[d_0, d_1, \dots, d_L]$ consists of a sequence of KAN layers
\[
    \Phi_{l} : \mathbb{R}^{d_l} \to \mathbb{R}^{d_{l+1}},
\qquad l = 0,\dots,L-1 .
\]
Let $\mathbf{x}^{(0)} = \mathbf{x}$ and $\mathbf{x}^{(L)} = \mathbf{y}$. The forward propagation through the network is defined recursively by
\[
    \mathbf{x}^{(l+1)} := \Phi_{l}(\mathbf{x}^{(l)}),
\qquad l = 0,\dots,L-1 .
\]
The function realised by the network is
\begin{equation}\label{eq:kan_composition}
    \mathbf{y} = \mathrm{KAN}(\mathbf{x}) := (\Phi_{L-1} \circ \Phi_{L-2} \circ \cdots \circ \Phi_0)(\mathbf{x}).
\end{equation}
\end{definition}
A schematic of a KAN of depth $L$ and shape $[d_0, d_1, \ldots, d_L]$ is given in Figure~\ref{fig:kan-schematic}. Each univariate function $\varphi_{q,p}^{(l)}$ of a KAN layer $\Phi_l$ is associated with the directed edge connecting input coordinate $p$ to output coordinate $q$. The blue nodes perform addition.

\begin{remark}[Connection to Kolmogorov--Arnold representation] Note that the Kolmogorov--Arnold representation \eqref{eq:ka-representation} of 
the function $f : [0,1]^{d} \to \mathbb{R}$ is exactly a KAN with depth $L = 2$ and shape $[d, 2d + 1, 1]$. Thus, the Kolmogorov--Arnold representation theorem can be equivalently interpreted as the statement that any continuous function on a compact domain can be realised by a KAN of appropriate finite width and depth.
\end{remark}

A key innovation in the KAN architecture is depth. Although the Kolmogorov--Arnold theorem allows a depth-$2$ representation, the involved functions may be rather non-smooth. By allowing a deeper architecture, one may find representations with smooth components. \cite{liu2025kan} highlight this point with the following example function:
\[
    f(x_1, x_2, x_3, x_4) = \exp(\sin(x_1^2 + x_2^2) + \sin(x_3^2 + x_4^2)).
\]
This function can be represented as a depth-$3$ KAN with shape $[4, 2, 1, 1]$ and smooth components. However, there may not exist a depth-$2$ KAN representing it with smooth activations.

In contrast to standard multilayer perceptrons, where scalar weights are followed by a fixed activation function, KANs replace scalar weights with learnable univariate functions, thereby applying nonlinearities prior to aggregation and yielding a functionally expressive architecture. Typically, each univariate function is parameterised using B-spline functions\footnote{We will not define B-splines here and refer the readers to the books \cite{de1978practical, schumaker2007spline} for their definition and properties.} defined on a finite grid, which enables controlled approximation properties. The orders of these B-spline functions control the smoothness of the resulting KAN, and the coefficients of the B-spline basis expansions are the trainable parameters. The grid sizes are hyperparameters that act as regularisers.

\subsection{A basic approximation theorem for KANs}\label{sec:kan-uat}
We consider functions that admit exact representations as compositions of KAN layers, where each univariate function is sufficiently smooth. The following theorem provides an error bound on the approximation of such functions by spline-based KANs.

\begin{theorem}[An approximation theorem for KANs, Theorem~2.1 of \cite{liu2025kan}]
\label{thm:kan}
Suppose that $f : \bbR^d \to \bbR$ admits a representation as a KAN of depth $L$:
\begin{equation}\label{eq:f_representation}
    f(\mathbf{x}) = (\Phi_{L-1} \circ \Phi_{L-2} \circ \cdots \circ \Phi_1 \circ \Phi_0)(\mathbf{x}),
\end{equation}
where every edge function $\varphi^{(l)}_{q,p}$ is $(k+1)$-times continuously differentiable.

Then there exists a constant $C_f > 0$, depending only on $f$ and the representation \eqref{eq:f_representation}, such that for any $G \in \mathbb{N}$, there exist spline-based KAN layers $\Phi_l^G$, $0 \le l \le L - 1$, made of $B$-spline functions of order $k  + 1$, satisfying, for every integer $q$ with $0 \le q \le k$,
\begin{equation}\label{eq:kat_bound}
    \left\|f - \Phi_{L-1}^G \circ \Phi_{L-2}^G \circ \cdots \circ \Phi_1^G \circ \Phi_0^G \right\|_{C^q} \;\le\; C_f \cdot G^{-k-1+q}.
\end{equation}
\end{theorem}
\begin{proof}
The proof is simple. Standard results in the theory of spline approximation (see, e.g., Section 6.4 of \cite{ schumaker2007spline}) tell us that there exist B-spline functions $\Phi^{(l), G}_{i,j}$ of order $k + 1$, constructed on a grid\footnote{Naturally, this grid has to be a partition of an interval covering the range of $(\Phi_{l-1} \circ \Phi_{l-2} \circ \cdots \circ \Phi_{1} \circ \Phi_{0})(\mathbf{x})$ as $\mathbf{x}$ varies over $[0, 1]^d$.} of mesh size $G^{-1}$, such that for any $0 \le q \le k$, 
\[
    \big\|(\Phi^{(l)}_{i,j} \circ \Phi_{l-1} \circ \cdots \circ \Phi_{1} \circ \Phi_{0})(\mathbf{x}) - (\Phi^{(l), G}_{i,j} \circ \Phi_{l-1} \circ \cdots \circ \Phi_{1} \circ \Phi_{0})(\mathbf{x})\big\|_{C^q} \le C_{l, i, j} \cdot G^{-k -1 + q},
\]
with a constant $C_{l, i, j}$ independent of $G$. We choose and fix these B-spline approximations in order of growing $l$. Then, for each $0 \le l \le L - 1$, one has that
\[
    R_l:= (\Phi^G_{L-1}\circ\cdots\circ\Phi^G_{l+1}\circ\Phi_{l}\circ\Phi_{l-1}\circ\cdots\circ\Phi_{0})(\mathbf{x}) -(\Phi_{L-1}^G\circ\cdots\circ\Phi_{l+1}^G\circ\Phi_{l}^G\circ\Phi_{l-1}\circ\cdots\circ\Phi_{0})(\mathbf{x})
\]
satisfies 
\[
    \|R_l\|_{C^q}\leq C_l \cdot G^{-k-1+q},
\]
with a constant $C_l$ that does not depend on $G$. Finally, one obtains \eqref{eq:kat_bound} by using the triangle inequality on the telescopic decomposition 
\[
    f - (\Phi^G_{L-1}\circ\Phi^G_{L-2}\circ\cdots\circ\Phi^G_{1}\circ\Phi^G_{0})(\mathbf{x})=R_{L-1}+R_{L-2}+\cdots+R_1+R_0. \qedhere
\]
\end{proof}

Theorem~\ref{thm:kan} shows that spline-based KANs with a sufficiently large grid size $G$ can approximate any function representable as a deep and smooth KAN quite well with a rate that is \emph{independent of the dimension} (which is natural since splines are being used to approximate univariate functions). Of course, the constant $C_f$ is dependent on the representation \eqref{eq:f_representation}; hence it potentially depends on the dimension $d$.

\begin{remark}
The recent work \cite{kratsios2025kolmogorov}  shows that residual KAN models can optimally approximate functions in Besov spaces $B^s_{p, q}(\mathcal{X})$ on bounded domains (or even fractal domains) $\mathcal{X} \subset \bbR^d$, with respect to weaker Besov norms $B^{\alpha}_{p, q}(\mathcal{X})$, $\alpha < s$.
\end{remark}

\section{Concluding remarks}\label{sec:conc}

Despite significant progress in understanding the expressivity of neural networks, many important questions remain open. We highlight a few broad directions.

While specific separation results demonstrate exponential benefits of depth for carefully constructed function classes, a comprehensive theory for general function classes, particularly in high dimensions, is still lacking. Precise quantification of depth--width trade-offs for smooth, compositional, or structured function classes remains an open challenge, and it is unclear to what extent existing results reflect fundamental phenomena versus artefacts of the particular constructions used.

Most theoretical results reviewed here focus on the feedforward architecture with ReLU or sigmoidal activations. The approximation-theoretic properties of convolutional networks (CNNs) and more modern architectures, such as transformers with attention mechanisms, normalisation layers, and graph neural networks, are far less well understood, and developing a comparably mature theory for these architectures is an important open problem.

The approximation theory of KANs is also in its early stages. While KANs offer an appealing alternative architectural paradigm grounded in the Kolmogorov--Arnold representation theorem, a rigorous and comprehensive comparison with classical FNNs in terms of approximation rates and parameter efficiency remains to be carried out.

Addressing these questions will not only deepen our theoretical understanding of neural networks but also inform the design of architectures that are both expressive and computationally efficient in practice.

\bibliographystyle{apalike}
\bibliography{refs.bib}
\end{document}